\documentclass[aos]{imsart}

\RequirePackage{amsthm,amsmath,amsfonts,amssymb}

\RequirePackage[authoryear]{natbib}
\RequirePackage[colorlinks,citecolor=blue,urlcolor=blue]{hyperref}
\RequirePackage{graphicx}
\usepackage{algorithm}
\usepackage{algorithmic}
\usepackage{pgfplots}
\usepackage{pgfplotstable}
\pgfplotsset{compat=1.18}
\usepackage{caption}
\usepackage{subcaption}
\usepackage{pgfplotstable}
\usepackage{tikz}
\usepackage{array}
\usepackage{booktabs}

\startlocaldefs
\theoremstyle{plain}

\newtheorem{theorem}{Theorem}[section]
\newtheorem{lemma}[theorem]{Lemma}

\newtheorem{proposition}{Proposition}

\theoremstyle{remark}
\newtheorem{definition}[theorem]{Definition}

\endlocaldefs

\begin{document}

\begin{frontmatter}
\title{Fréchet regression with implicit denoising and multicollinearity reduction}
\runtitle{A sample running head title}

\begin{aug}
\author[A]{\fnms{Dou El Kefel}~\snm{Mansouri}\ead[label=e1]{douelkefel.mansouri@univ-tiaret.dz}\orcid{0000-0001-7365-4804}},
\author[B]{\fnms{Seif-Eddine}~\snm{Benkabou}\ead[label=e2]{seif.eddine.benkabou@
univ-poitiers.fr}\orcid{0000-0002-4526-534X}}
\and
\author[C]{\fnms{Khalid}~\snm{ Benabdeslem}\ead[label=e3]{khalid.benabdeslem@univ-lyon1.fr} \orcid{0000-0002-4324-924X}}
\address[A]{Department of Biology, Ibn
Khaldoun University, Tiaret 14000, Algeria\printead[presep={,\ }]{e1}}

\address[B]{
Laboratory of Computer Science and
Automatic Control for Systems (LIAS)/École Nationale Supérieure de
Mécanique et d'Aérotechnique Poitiers Futuroscope (ISAE-ENSMA), University of Poitiers, 86000 Poitiers, France\printead[presep={,\ }]{e2}}

\address[C]{Laboratoire d'InfoRmatique en Image et
Systèmes d'information (LIRIS), University of Lyon 1, 69622 Villeurbanne, France\printead[presep={,\ }]{e3}}
\end{aug}

\begin{abstract}
Fréchet regression extends linear regression to model complex responses in metric spaces, making it particularly relevant for multi-label regression, where each instance can have multiple associated labels. However, addressing noise and dependencies among predictors within this framework remains underexplored. In this paper, we present an extension of the Global Fréchet regression model that enables explicit modeling of relationships between input variables and multiple responses. To address challenges arising from noise and multicollinearity, we propose a novel framework based on implicit regularization, which preserves the intrinsic structure of the data while effectively capturing complex dependencies. Our approach ensures accurate and efficient modeling without the biases introduced by traditional explicit regularization methods. Theoretical guarantees are provided, and the performance of the proposed method is demonstrated through numerical experiments.
\end{abstract}

\begin{keyword}[class=MSC]
\kwd[Primary ]{00X00}
\kwd{00X00}
\kwd[; secondary ]{00X00}
\end{keyword}

\begin{keyword}
\kwd{Global Fréchet regression}
\kwd{Multi-label}
\kwd{Denoising}
\kwd{Implicit regularization}
\end{keyword}

\end{frontmatter}

\section{Introduction}
Fréchet regression, an extension of classical linear regression to general metric spaces, offers a robust framework for modeling complex relationships between variables when the responses lie outside of Euclidean spaces. This approach is especially well suited to high-dimensional datasets, such as vector representations, with particular relevance to fields like imaging, where capturing nonlinear dependencies and the intrinsic data structure is critical for accurate modeling (\citet{frechet1948elements}, \citet{petersen2019frechet}, \citet{bhattacharjee2023single}, \citet{qiu2024random}). A significant consideration in Fréchet regression arises when predicting multiple responses simultaneously, as seen in multi-target or multidimensional problems (\citet{zhang2007ml}, \citet{hyvonen2024multilabel}). Unlike traditional regression, where each observation corresponds to a single response, Fréchet regression can be extended to model complex interactions between multiple outputs. This ability to address complex relationships between several responses opens new avenues, particularly in fields such as bioinformatics (\citet{huang2005detection}) and image analysis (\citet{lathuiliere2019comprehensive}), where multidimensional data and interdependencies between responses require adaptive and specialized methodologies. However, to date, the handling of multilabel scenarios within the context of Fréchet regression remains relatively unexplored in the literature, despite its potential significance in addressing complex, multidimensional applications.

In this paper, we present an extension of the Global Fréchet regression model, a specific variant of Fréchet regression that generalizes classical multiple linear regression by modeling responses as random objects. This extension enables the explicit modeling of relationships between input variables and multiple responses, thereby addressing the multi-label setting.

Our second contribution addresses the critical challenges of noise and multicollinearity within the proposed Fréchet regression framework. High-dimensional data (\citet{zhang2023dimension}), prevalent in fields such as imaging and bioinformatics, often suffer from noisy inputs and responses, as well as multicollinearity among predictors. These issues impede accurate modeling of complex relationships, increasing risks of overfitting and inefficiency. To overcome these challenges, we propose a novel denoising and multicollinearity reduction framework based on implicit regularization. This approach preserves the data's intrinsic structure while capturing intricate dependencies between predictors and responses. Unlike traditional methods relying on explicit regularization, which introduce bias and complicate optimization \citep{de2005learning}, our framework circumvents these limitations with a smoother, more efficient regularization mechanism. As a result, it ensures precision, robustness, and computational efficiency in modeling high-dimensional relationships.

The main contributions of this paper are as follows:
\begin{itemize}
    \item \textit{Extension of Fréchet Regression}: We present an extension of the Global Fréchet regression model that enables explicit modeling of relationships between input variables and multiple responses.
    \item \textit{Denoising and Multicollinearity Reduction}: We propose a novel framework based on implicit regularization to address noise and multicollinearity challenges. This method preserves the intrinsic structure of the data while effectively capturing the complex dependencies between variables and multiple responses, ensuring accurate and efficient modeling.
\end{itemize}

\section{Proposed Method}
\subsection{Preliminaries}
We use boldface notation to represent vectors and matrices. Let \((\Omega, d)\) be a metric space with a specific metric $d$. We assume that \(\Omega\) is complete, which ensures the convergence of Cauchy sequences within this space. Let  $\mathbf{X} \in \mathbb R_{}^{n \times m}$ denote the multivariate random variables, and let $\mathbf{Y}:\Omega \rightarrow	\mathbb R_{}^{n \times q}$ be complex random objects. We consider a random process \((\mathbf{X}, \mathbf{Y}) \sim F\), with $F$ denoting the joint distribution of \((\mathbf{X}, \mathbf{Y})\).
 \(F_{\mathbf{X}}\) and \(F_{\mathbf{Y}}\) are the marginal distributions of \(\mathbf{X}\) and \(\mathbf{Y}\), respectively. \(F_{\mathbf{X}|\mathbf{Y}}\) and \(F_{\mathbf{Y}|\mathbf{X}}\) are the conditional distributions that should be assumed to exist and to be well defined. Let $\mu = E(\mathbf{X})$ and $\Sigma= Var(\mathbf{X})$.
Let $w_\oplus$ and $V_\oplus$, the Fréchet mean and Fréchet variance of random objects in metric spaces (\citet{frechet1948elements}), be defined as extensions of the traditional concepts of mean 
\begin{equation*}
\text{\hspace{-2.3cm}and variance:\hspace{1cm}}    w_\oplus = \arg \min_{w \in \Omega} E(d^2(\mathbf{Y}, w)), \hspace{1cm}
 V_\oplus = E(d^2(\mathbf{Y}, w_\oplus)).
\end{equation*}

\subsection{Global Frechet regression}

\begin{definition}[Fréchet regression]
According to \citet{petersen2019frechet},  Fréchet regression can be  defined as: 
\begin{equation}
m_\oplus(x) =\arg \min_{w \in \Omega} = M_\oplus(w,x), \hspace{1cm} M_\oplus(.,x)=E(d^2(\mathbf{Y}, .)|\mathbf{X}=x)
\end{equation}
where,  $M_\oplus(.,x)$ is the (conditional) Fréchet function.
\end{definition}
Global Fréchet regression is a special case of Fréchet regression that extends classical multiple linear regression to accommodate responses that are random objects (\cite{petersen2019frechet}). More precisely, the standard linear regression function  can be interpreted as the solution of an optimization problem and is defined as follows:
\begin{equation}
m(x) =\arg \min_{y \in \Omega=\mathcal{R}}  E(s(\mathbf{X},x)d_{E}^2(\mathbf{Y}, y))
\end{equation}
where, $m(x)$ is the argument that minimizes $E$, and $d_{E}$ is the standard Euclidean metric. $s$  is the weight function chosen to fit the nature of the responses and may have different characteristics from the weights used in local regressions. It is defined as follows: 
\begin{equation*}
\label{x}
s(\mathbf{X},x) =    1 + (\mathbf{X} - \mu)^T \Sigma^{-1} (x - \mu)
\end{equation*}

\begin{definition} [Global Fréchet regression]
Replacement of the Euclidean metric $d_{E}$ with a more general metric $d$ of $\Omega$ involves redefining the regression function as a minimization of a generalized loss function. Thus, Global Fréchet regression can be  defined as: 
\begin{equation}
\hat{m}_\oplus(x) =\arg \min_{w \in \Omega} M(w,x), \hspace{1cm} M(.,x)=E\left[  s(\mathbf{X},x)d^2(\mathbf{Y}, .)\right]
\end{equation}
\end{definition}
This leads to 
\begin{equation}
\label{GFr}
\hat{m}_\oplus(x) =\arg \min_{w \in \Omega}  \sum_{i=1}^{n}\left[1+(x_{i}-\overline{\mathbf{X}})\widehat{\Sigma}^{-1}(x_{i}-\overline{\mathbf{X}})^T\right]d^2(\mathbf{Y} - \omega)
\end{equation}
 
where $\overline{\mathbf{X}}=n^{-1}\sum_{i=1}^{n} x_{i}$, and
$\widehat{\Sigma}=n^{-1}\sum_{i=1}^{n}(x_{i}-\overline{\mathbf{X}})(x_{i}-\overline{\mathbf{X}})^T$ are the empirical estimates of $\mu$ and $\Sigma$ in \ref{x}, respectively.

\subsection{Multilabel Global Fréchet regression}
Equation \eqref{GFr}, in its standard form, is not inherently designed for the multilabel context, as it focuses primarily on a single response. In this paper, we introduce a new term into the same equation, enabling the explicit modeling of relationships between input variables and multiple responses, thus addressing the multilabel setting.

\begin{proposition}\label{P1}
The Multilabel Global Fréchet regression can be defined as:
\begin{equation}
\label{MGFr}
\widehat{\zeta}_\oplus(x) =\arg \min_{w \in \Omega}  \sum_{i=1}^{n}\left[1+(x_{i}-\overline{\mathbf{X}})\widehat{\Sigma}^{-1}(x_{i}-\overline{\mathbf{X}})^T+(x_{i}-\overline{\mathbf{X}}) \widehat{\Sigma}_{\mathbf{X}\mathbf{Y}}^{-1} (y_{i}-\overline{\mathbf{Y}})^T
\right]d^2(\mathbf{Y} - \omega)
\end{equation}

where $(x_{i}-\overline{\mathbf{X}}) \widehat{\Sigma}_{\mathbf{X}\mathbf{Y}}^{-1} (y_{i}-\overline{\mathbf{Y}})^T$ is used to consider the cross-relationships between predictors and responses. 
\end{proposition}
Next, we show that $\widehat{\zeta}_\oplus(x)$ is a consistent estimator, meaning it converges in probability to the true parameter $w_{0}$ as the sample size $n$ tends to infinity. To establish this, we require 
that $\{(x_{i},y_{i})_{i=1}^n\}$ are independent and identically distributed and a linear relationship between $\mathbf{X}$ and $\mathbf{Y}$ exists, such that the residuals having zero expectation. 
We also assume that $\widehat{\Sigma}$ and $\widehat{\Sigma}_{\mathbf{X}\mathbf{Y}}$ are positive definite and invertible. The Multilabel Global Fréchet regression  is detailed in  Algorithm \ref{alg:zeta_oplus}.

\begin{algorithm}
\caption{Estimation of \(\widehat{\zeta}_\oplus(x)\)}
\label{alg:zeta_oplus}
\begin{algorithmic}[1]
    \STATE \textbf{Input:} $\mathbf{X} \in \mathbb R_{}^{n \times m}$, $\mathbf{Y} \in	\mathbb R_{}^{n \times q}$, $\Omega$, \(\epsilon\), \(K_{\text{max}}\).
    \STATE \textbf{Compute:}\\
     \quad $\overline{\mathbf{X}}$, $\overline{\mathbf{Y}}$, $\widehat{\Sigma}$, $\widehat{\Sigma}_{\mathbf{X}\mathbf{Y}}$, $\widehat{\Sigma}^{-1}$, $\widehat{\Sigma}_{\mathbf{X}\mathbf{Y}}^{-1}$
    \STATE \textbf{Optimization:}\\
    \quad \(\text{Initialize } \omega \text{ to } \omega_0 \in \Omega\).\\
    \quad \(k \leftarrow 0\).
    \STATE \quad \textbf{Repeat:}\\
    \quad \text{Compute } 
     $f(w)=  
    \sum_{i=1}^{n}\left[1+(x_{i}-\overline{\mathbf{X}})\widehat{\Sigma}^{-1}(x_{i}-\overline{\mathbf{X}})^T+(x_{i}-\overline{\mathbf{X}}) \widehat{\Sigma}_{\mathbf{X}\mathbf{Y}}^{-1} (y_{i}-\overline{\mathbf{Y}})^T
\right] d^2(\mathbf{Y} - \omega)$\\
    \quad \text{Update } $\omega$, 
    \quad \(k \leftarrow k + 1\) \quad
     \textbf{Until:} \(\|\omega_{k} - \omega_{k-1}\| < \epsilon\) \textbf{ or } \(k \geq K_{\text{max}}\).
    \STATE \textbf{Output:} $\widehat{\zeta}_\oplus(x)$.
\end{algorithmic}
\end{algorithm}

\begin{lemma}[Consistency]\label{l2}
Under the assumptions that $\widehat{\Sigma}$ and $\widehat{\Sigma}_{\mathbf{X}\mathbf{Y}}$ are positive definite and invertible for all sufficiently large $n$, and that $d^2(\mathbf{Y} - \omega)$ is a non-negative, continuous function bounded by a constant $M > 0$ for all $\omega \in \Omega$. Then, $\widehat{\zeta}_\oplus(x) \xrightarrow{p} \zeta_\oplus(x)$ if $n \to \infty$.    \end{lemma}

\begin{proof}[Proof of Lemma \ref{l2}]
As $\overline{\mathbf{X}} \xrightarrow{p} \mu$ and $\widehat{\Sigma} \xrightarrow{p} \Sigma$, by the law of large numbers $\widehat{\Sigma}^{-1} \xrightarrow{p} \Sigma^{-1}$. Also, we have  $\widehat{\Sigma}_{\mathbf{X},\mathbf{Y}} \xrightarrow{p} \Sigma_{\mathbf{X},\mathbf{Y}}$ and $\widehat{\Sigma}^{-1}_{\mathbf{X},\mathbf{Y}} \xrightarrow{p} \Sigma^{-1}_{\mathbf{X},\mathbf{Y}}$. Under the assumption that $\widehat{\Sigma}_{\mathbf{X},\mathbf{Y}}$ is positive definite and invertible, these convergences are ensured. Now, let:

$f(w)=  
    \sum_{i=1}^{n}\left[1+(x_{i}-\overline{\mathbf{X}})\widehat{\Sigma}^{-1}(x_{i}-\overline{\mathbf{X}})^T+(x_{i}-\overline{\mathbf{X}}) \widehat{\Sigma}_{\mathbf{X}\mathbf{Y}}^{-1} (y_{i}-\overline{\mathbf{Y}})^T
\right] d^2(\mathbf{Y} - \omega)$

where, $d^2(\mathbf{Y} - \omega)$ is a non-negative and continuous function bounded by a constant $M$. This ensures that the distance is well-defined for any $\omega \in \Omega$.

$f_{true}(w)=  
    \sum_{i=1}^{n}\left[1+(x_{i}-\mu)\widehat{\Sigma}^{-1}(x_{i}-\mu)^T+(x_{i}-\mu) \widehat{\Sigma}_{\mathbf{X}\mathbf{Y}}^{-1} (y_{i}-\mu \mathbf{Y})^T
\right] d^2(\mathbf{Y} - \omega)$
As the empirical matrices converge to the true matrices and $d^2(\mathbf{Y} - \omega)$ is in $\omega$ then $f(w) \xrightarrow{p} f_\text{true}(w)$. 
\[\hspace{-1.4cm}\text{Now, as $\Omega$ is complete, } \widehat{\zeta}_\oplus(x) = \arg \min_{w \in \Omega}  f(w) \xrightarrow{p} \arg \min_{w \in \Omega} f_\text{true}(w) = \zeta_\oplus(x)\]
\end{proof}

\begin{lemma}[Rate of Convergence]\label{RateConv}
    \(\widehat{\zeta}_\oplus(x)\) converges at a rate of \( O_p(n^{-1/2}) \)
\end{lemma}
\begin{proof}[Proof of Lemma \ref{RateConv}]
To evaluate the convergence rate, it is necessary to examine the deviation between $f(w)$ and $f_\text{true}(w)$. Since \(\widehat{\Sigma} \xrightarrow{p} \Sigma\) and \(\widehat{\Sigma}_{\mathbf{X}\mathbf{Y}} \xrightarrow{p} \Sigma_{\mathbf{X}\mathbf{Y}}\), by the law of large numbers, we have:

\[
\|\widehat{\Sigma}^{-1} - \Sigma^{-1}\| = O_p(n^{-1}) \text{  and  } \|\widehat{\Sigma}_{\mathbf{X}\mathbf{Y}}^{-1} - \Sigma_{\mathbf{X}\mathbf{Y}}^{-1}\| = O_p(n^{-1}) 
\]
Thus, 
\begin{equation*}
\begin{aligned}
\|f(w) - f_\text{true}(w)\| &= \left\| \sum_{i=1}^{n} \left[ (x_i - \overline{\mathbf{X}}) (\widehat{\Sigma}^{-1} - \Sigma^{-1}) (x_i - \overline{\mathbf{X}})^T \right. \right. \\
& \quad + \left. \left. (x_i - \overline{\mathbf{X}})(\widehat{\Sigma}_{\mathbf{X}\mathbf{Y}}^{-1} - \Sigma_{\mathbf{X}\mathbf{Y}}^{-1}) (y_i - \overline{\mathbf{Y}})^T \right] d^2(\mathbf{Y} - \omega) \right\| 
\end{aligned}
\end{equation*}
Leveraging the properties of covariance matrices and considering that \(d^2(\mathbf{Y} - \omega)\) is a non-negative, continuous function bounded by a constant \(M\), it follows that: \[
\|f(w) - f_\text{true}(w)\| = O_p(n^{-1})
\]
Given the continuity of the functions \( f(w) \) and \( f_\text{true}(w) \), as well as the fact that the error in the covariance matrices converges at a rate of \( O_p(n^{-1}) \), the error associated with the estimator \(\widehat{\zeta}_\oplus(x)\) can be described by:

\[
\|\widehat{\zeta}_\oplus(x) - \zeta_\oplus(x)\| = O_p(n^{-1/2}).
\]

In other words, if the covariance matrices converge at a rate of \( O_p(n^{-1}) \), then the estimator \(\widehat{\zeta}_\oplus(x)\) converges at a rate of \( O_p(n^{-1/2}) \), which is characteristic of asymptotically normal estimators.
\end{proof}

\begin{lemma}[Asymptotic Normality]\label{asymptotic}
    $\widehat{\zeta}_\oplus(x)$ is asymptotically normal.
\end{lemma}

\begin{proof}[Proof of Lemma \ref{asymptotic}]
Under the assumption that  there are sufficient regularity conditions on \(d^2(\mathbf{Y} - \omega)\) so that the function \(f(w)\) is well-behaved around the minimum and that the matrices $\widehat{\Sigma}^{-1}$, $\widehat{\Sigma}_{\mathbf{X}\mathbf{Y}}^{-1}$ are positive definite and invertible, while assuming that the conditional distributions \(F_{\mathbf{X}}\) and \(F_{\mathbf{Y}}\) are well-defined and satisfy the usual regularity conditions, the consistency of the estimator is established in the proof \ref{l2}. Therefore, we can now proceed to the demonstration of the Normal Asymptotic.

Let's use a Taylor expansion of the objective function around the true parameter \(\zeta_\oplus(x)\). 
\[
f(w) \approx f(\zeta_\oplus(x)) + \nabla f(\zeta_\oplus(x))^T (w - \zeta_\oplus(x)) + \frac{1}{2} (w - \zeta_\oplus(x))^T \nabla^2 f(\zeta_\oplus(x)) (w - \zeta_\oplus(x)),
\]
where \(\nabla f(\zeta_\oplus(x))\) is the gradient and \(\nabla^2 f(\zeta_\oplus(x))\) is the Hessian. At the optimum, $\nabla f(\widehat{\zeta}_\oplus(x)) = 0$. By applying the Taylor expansion, we obtain:
\[
0 \approx \nabla f(\zeta_\oplus(x)) + \nabla^2 f(\zeta_\oplus(x)) (\widehat{\zeta}_\oplus(x) - \zeta_\oplus(x)).
\]
This involves:\[
\sqrt{n} (\widehat{\zeta}_\oplus(x) - \zeta_\oplus(x)) \approx -\sqrt{n} \nabla^2 f(\zeta_\oplus(x))^{-1} \nabla f(\zeta_\oplus(x)).
\]
By the central limit theorem,\[
\sqrt{n} \nabla f(\zeta_\oplus(x)) \xrightarrow{d} \mathcal{N}(0, \Sigma_{\zeta}),
\]
where \(\Sigma_{\zeta}\) is the covariance matrix of \(\nabla f(\zeta_\oplus(x))\).
\[\text{Then, }
\sqrt{n} (\widehat{\zeta}_\oplus(x) - \zeta_\oplus(x)) \xrightarrow{d} \mathcal{N}\left(0, \nabla^2 f(\zeta_\oplus(x))^{-1} \Sigma_{\zeta} \nabla^2 f(\zeta_\oplus(x))^{-1}\right).
\]
Therefore, $\widehat{\zeta}_\oplus(x)$ is asymptotically normal.
\end{proof}

\subsection{Denoising and Multicollinearity Reduction for Multilabel Global Fréchet regression}
Equation \eqref{MGFr} is subject to two major challenges: noise in the data \(\mathbf{X}\) and/or \(\mathbf{Y}\), as well as multicollinearity between variables. To address these challenges, this paper proposes a denoising and multicollinearity reduction method that relies on implicit regularization. This approach replaces traditional explicit regularization techniques, which can introduce biases, and offers a more subtle and efficient alternative.
\begin{proposition}\label{P2}
The Denoising and Multicollinearity Reduction for Multilabel Global Fréchet Regression can be defined as:
\begin{equation}
\label{VSMGFr}
\widehat{\Lambda}_\oplus(x) =\arg \min_{w \in \Omega}  \sum_{i=1}^{n}\left[1+(x_{i}-\overline{\mathbf{X}})\widehat{\Sigma}^{-1}(x_{i}-\overline{\mathbf{X}})^T+(x_{i}-\overline{\mathbf{X}}) \widehat{\Theta} (y_{i}-\overline{\mathbf{Y}})^T
\right]d^2(\mathbf{Y} - \omega)
\end{equation}
where, $\widehat{\Theta}$ is the component responsible for denoising and multicollinearity reduction.    
\end{proposition}
 In the following, we demonstrate the consistency of the new formulation.

\begin{lemma}[Multicollinearity and Impact of Noise] \label{l3}
$\mathbf{X}$ and/or $\mathbf{Y}$ are noisy, as well as multicollinearity among the variables is present, implies that $\widehat{\Lambda}_\oplus(x)$ is affected.     
\end{lemma}
\begin{proof}[Proof of Lemma \ref{l3}]Noise in $\mathbf{X}$ and/or $\mathbf{Y}$ disrupt the estimation of $\Sigma^{-1}$. Consequently, this can lead to biased estimation, as the chosen model may end up fitting the noisy data and the inaccurately estimated covariance structure, rather than capturing the true underlying relationships. Furthermore, multicollinearity between variables leads to an ill-conditioned inverse covariance matrix, which can exacerbate noise problems by making estimates even more sensitive to fluctuations in the data. Direct denoising $\mathbf{X}$ and/or $\mathbf{Y}$ can be challenging due to the complex nature of noise. In contrast, denoising $\Sigma^{-1}$ can often be more effective. In fact, $\Sigma^{-1}$ plays a crucial role in providing key information about the structure and dependencies of the data. Thus, improving $\Sigma^{-1}$ leads to better assessment and compensation of the effects of noise on the relationships between variables and also mitigate the effects of multicollinearity. 

In our proposal, we have chosen to focus our efforts only on improving $\widehat{\Sigma}_{\mathbf{X}\mathbf{Y}}^{-1}$ rather than simultaneously improving $\widehat{\Sigma}^{-1}$ and $\widehat{\Sigma}_{\mathbf{X}\mathbf{Y}}^{-1}$. This is based on two main considerations. First, \(\widehat{\Sigma}_{\mathbf{X}\mathbf{Y}}^{-1}\) plays a crucial role in estimating the relationships between \(\mathbf{X}\) and \(\mathbf{Y}\) and is also the main source of uncertainty in this estimation. By focusing only on improving \(\widehat{\Sigma}_{\mathbf{X}\mathbf{Y}}^{-1}\), we can more specifically address the noise and multicollinearity issues. This leads to an overall improvement in the performance of the \(\widehat{\Lambda}_\oplus(x)\), stabilizing the predictions and reducing the uncertainty related to the covariance structure between \(\mathbf{X}\) and \(\mathbf{Y}\). Second, focusing solely on \(\widehat{\Sigma}_{\mathbf{X}\mathbf{Y}}^{-1}\) avoids the complexity and redundant efforts involved in improving both matrices simultaneously. 

Furthermore, the proposed estimator preserves the original data $\mathbf{X}$ and $\mathbf{Y}$ while integrating a reduced version of $\widehat{\Sigma}_{\mathbf{X}\mathbf{Y}}^{-1}$ into the  $\widehat{\Lambda}_\oplus(x)$. The objective is to manage noise and multicollinearity without altering the original data. In the following, we provide theoretical and empirical evidence that supports the effectiveness of improving \(\widehat{\Sigma}_{\mathbf{X}\mathbf{Y}}^{-1}\) in mitigating noise and multicollinearity issues, thus justifying our targeted approach.
\end{proof}

\begin{lemma}[Invertibility, Sparsity and Asymptotic Normality] \label{l4}
  Replacing  $\widehat{\Sigma}_{\mathbf{X}\mathbf{Y}}^{-1}$ with a new reduced matrix  $\widehat{\Theta}$ significantly improves $\widehat{\Lambda}_\oplus(x)$.  
\end{lemma}
\begin{proof}[Proof of Lemma \ref{l4}]
Let $\widehat{\Sigma}_{\mathbf{X}\mathbf{Y}}^{-1}$ be the estimated inverse covariance matrix and $\widehat{\Theta}$ be a reduced approximation obtained via singular value decomposition (SVD). More precisely, we have $\widehat{\Sigma}_{\mathbf{X}\mathbf{Y}}^{-1} = U \mathbf{S} V^T$, where $U$ and $V$ are orthogonal matrices and $\mathbf{S}$ is a diagonal matrix containing the singular values. By applying a threshold $\tau$ to filter out small singular values, we obtain a reduced approximation $\widehat{\Theta} = U \mathbf{S}^* V^T$, where $\mathbf{S}^*$ is a truncated version of $\mathbf{S}$. This approach allows to simplify the matrix while preserving the most significant features of the original structure. We now prove the  properties of invertibility, sparsity, and asymptotic normality.

\textbf{(i) Invertibility:}
To ensure the invertibility of $\widehat{\Theta}$ while preserving its sparsity, we apply Tikhonov regularization, adding a term ($\epsilon \times I$) to $\mathbf{S}^*$, where $I$ is the identity matrix and $\epsilon$ is a very small positive number. Let \(\mathbf{S} = \text{diag}(\sigma_1, \sigma_2, \ldots, \sigma_k)\) be the matrix of singular values of \(\widehat{\Sigma}_{\mathbf{X}\mathbf{Y}}^{-1}\) and \(\mathbf{S}^* = \text{diag}(\sigma_1^*, \sigma_2^*, \ldots, \sigma_k^*)\) be the truncated version of \(\mathbf{S}\). Applying Tikhonov regularization, we obtain:
\[
\mathbf{S}^* + (\epsilon \times I) = \text{diag}(\sigma_1^* + \epsilon, \sigma_2^* + \epsilon, \ldots, \sigma_k^* + \epsilon).
\]

Since \(\epsilon > 0\) and all singular values \(\sigma_i^*\) are positive, every element of \(\mathbf{S}^* + (\epsilon \times I)\) is strictly positive. Thus, the diagonal matrix \(\mathbf{S}^* + (\epsilon \times I)\) is positive definite. Consequently, \(\widehat{\Theta}_{\text{reg}} = U (\mathbf{S}^* + (\epsilon \times I) V^T\) is also invertible, thanks to the Tikhonov regularization which ensures that all eigenvalues are strictly positive, regardless of the initial truncation.

\textbf{(ii) Sparsity:} The difference between $\widehat{\Theta}$ and $\widehat{\Sigma}_{\mathbf{X}\mathbf{Y}}^{-1}$ can be evaluated in terms of the Frobenius norm: \[ \|\widehat{\Theta} - \widehat{\Sigma}_{\mathbf{X}\mathbf{Y}}^{-1}\|_F = \|U \mathbf{S}^* V^T - U \mathbf{S} V^T\|_F.
\]
Using the Frobenius norm property for matrices, we have, 
$\|\widehat{\Theta} - \widehat{\Sigma}_{\mathbf{X}\mathbf{Y}}^{-1}\|_F = \|\mathbf{S}^* - \mathbf{S}\|_F.$
Since $\mathbf{S}^*$ is obtained by truncating the small singular values of $\mathbf{S}$, the gap $\|\mathbf{S}^* - \mathbf{S}\|_F$ is dominated by the sum of the truncated singular values:$
\|\mathbf{S}^* - \mathbf{S}\|_F \leq \sum_{i \in \mathcal{T}} \sigma_i,$
where $\mathcal{T}$ is the set of indices corresponding to the small truncated values. Thus, if $\epsilon$ is sufficiently small, the gap between $\widehat{\Theta}$ and $\widehat{\Sigma}_{\mathbf{X}\mathbf{Y}}^{-1}$ is relatively small, i.e.:$
\|\widehat{\Theta} - \widehat{\Sigma}_{\mathbf{X}\mathbf{Y}}^{-1}\|_F \leq \epsilon.$

\textbf{(iii) Asymptotic Normality:}
We show that the estimator \(\widehat{\Lambda}_\oplus(x)\) based on $\widehat{\Theta}_{\text{reg}}$ preserves its asymptotic normality. Assuming that $\widehat{\Lambda}_\oplus(x)$ is an estimator for a parameter $\Lambda_\oplus(x)$ with asymptotic variance $\Sigma_{\Lambda}$, we have,
\[
\sqrt{n}(\widehat{\Lambda}_\oplus(x) - \Lambda_\oplus(x)) \xrightarrow{d} \mathcal{N}(0, \Sigma_{\Lambda}).
\]
Using SVD decomposition and Tikhonov regularization, the gap between $\widehat{\Theta}_{\text{reg}}$ and $\widehat{\Sigma}_{\mathbf{X}\mathbf{Y}}^{-1}$ is controlled and thus does not compromise the asymptotic normality of the estimator. Let us formulate this precision: \[
\|\widehat{\Theta}_{\text{reg}} - \widehat{\Sigma}_{\mathbf{X}\mathbf{Y}}^{-1}\|_F \leq \epsilon.
\]
If $\widehat{\Sigma}_{\mathbf{X}\mathbf{Y}}^{-1}$ is a good approximation of $\Sigma_{\mathbf{X}\mathbf{Y}}^{- 1}$, then the total approximation error is dominated by the estimation error of $\widehat{\Sigma}_{\mathbf{X}\mathbf{Y}}^{-1}$ : \[
\|\widehat{\Theta} - \Sigma_{\mathbf{X}\mathbf{Y}}^{-1}\|_F \leq \|\widehat{\Sigma}_{\mathbf{X}\mathbf{Y}}^{-1} - \Sigma_{\mathbf{X}\mathbf{Y}}^{-1}\|_F + \epsilon.
\]
Thus, $\widehat{\Theta}_{\text{reg}}$ is also a good approximation of $\Sigma_{\mathbf{X}\mathbf{Y}}^{-1}$ and preserves the asymptotic normality of the estimator.
\end{proof}
Overall, this approach simplifies $\widehat{\Sigma}_{\mathbf{X}\mathbf{Y}}^{-1}$  by reducing its complexity while retaining the most important aspects of its structure.
The process of estimating $\widehat{\Theta}$ is presented in  Algorithm \ref{alg:calc_theta}. 
Also, the estimation of $\widehat{\Lambda}_\oplus(x)$ is detailed in Algorithm \ref{alg:Lambda_oplus}.

\begin{algorithm}
\caption{Calculation of $\widehat{\Theta}$ (implicit denoising and multicollinearity mitigation)}
\label{alg:calc_theta}
\begin{algorithmic}[1]
    \STATE \textbf{Input:} Matrix $\widehat{\Sigma}_{\mathbf{X}\mathbf{Y}}^{-1}$,  $\tau$, $\epsilon$.
    \STATE Perform SVD on $\widehat{\Sigma}_{\mathbf{X}\mathbf{Y}}^{-1}$: $[U, \mathbf{S}, V] = \text{SVD}(\widehat{\Sigma}_{\mathbf{X}\mathbf{Y}}^{-1})$
    \STATE Initialize $\mathbf{S}= \mathbf{S^*}$
    \FOR{each singular value $\sigma$ in $\mathbf{S^*}$}
        \IF{$\sigma < \tau$}
            \STATE $\sigma = 0$
        \ENDIF
    \ENDFOR
    \STATE Apply Tikhonov Regularization to $\mathbf{S^*}$: $\mathbf{S^*} = \mathbf{S^*} + (\epsilon \times I)$
    \STATE Reconstruct the simplified matrix: $\widehat{\Theta} = U \mathbf{S^*} V^T$
    \STATE \textbf{Output:} $\widehat{\Theta}$
\end{algorithmic}
\end{algorithm}

\begin{algorithm}
\caption{Estimation of $\widehat{\Lambda}_\oplus(x)$}
\label{alg:Lambda_oplus}
\begin{algorithmic}[1]
    \STATE \textbf{Input:} $\mathbf{X} \in \mathbb{R}^{n \times m}$, $\mathbf{Y} \in \mathbb{R}^{n \times q}$, $\Omega$, \(\epsilon\), \(K_{\text{max}}\), Threshold \(\tau\).
    \STATE \textbf{Compute:}\\
     \quad $\overline{\mathbf{X}}$, $\overline{\mathbf{Y}}$, $\widehat{\Sigma}$, $\widehat{\Sigma}_{\mathbf{X}\mathbf{Y}}$, $\widehat{\Sigma}^{-1}$, $\widehat{\Sigma}_{\mathbf{X}\mathbf{Y}}^{-1}$
    \STATE \textbf{Calculate $\widehat{\Theta}$:}\\
     \quad \textbf{Call} \texttt{Algorithm \ref{alg:calc_theta}} \textbf{with inputs } $\widehat{\Sigma}_{\mathbf{X}\mathbf{Y}}^{-1}$ \textbf{ and } \(\tau\).
    \STATE \textbf{Optimization:}\\
    \quad \(\text{Initialize } \omega \text{ to  } \omega_0 \in \Omega\).\\
    \quad \(k \leftarrow 0\).
    \STATE \quad \textbf{Repeat:}\\
    \quad \text{Compute } 
     $f(w)=  
    \sum_{i=1}^{n}\left[1+(x_{i}-\overline{\mathbf{X}})\widehat{\Sigma}^{-1}(x_{i}-\overline{\mathbf{X}})^T+(x_{i}-\overline{\mathbf{X}})\widehat{\Theta} (y_{i}-\overline{\mathbf{Y}})^T
\right] d^2(\mathbf{Y} - \omega)$\\
    \quad \text{Update } $\omega$,
    \quad \(k \leftarrow k + 1\) \quad
     \textbf{Until:} \(\|\omega_{k} - \omega_{k-1}\| < \epsilon\) \textbf{ or } \(k \geq K_{\text{max}}\).
    \STATE \textbf{Output:} $\widehat{\Lambda}_\oplus(x)$ 
\end{algorithmic}
\end{algorithm}

\section{Simulation Studies}
\subsection{\textbf{Multilabel Fréchet Regression for Probability Distributions}}
\subsubsection*{\textbf{Computation details}}
The explanatory variables \(\mathbf{X}\) are simulated from a multivariate normal distribution, \(\mathbf{X} \sim \mathcal{N}(\mu_\mathbf{X}, \Sigma_\mathbf{X})\), where \(\mu_\mathbf{X} \in \mathbb{R}^m\) represents the vector of variable means, defined by \(\mu_\mathbf{X} = [0, 1, -1, 2, 0]\). The covariance matrix \(\Sigma_\mathbf{X} \in \mathbb{R}^{m \times m}\) is positive definite and invertible, constructed as the product of a randomly generated matrix and its transpose. The responses \(\mathbf{Y}\) are modeled via \(\mathbf{Y} \sim \mathcal{N}(\beta(x_i) \mathbf{1}_q, \Sigma_\mathbf{Y})\), where \(\beta(x_i)\) represent the conditional means of the response variables, defined as \(\beta(x_i) \sim \mathcal{N}(\beta_0 + \delta x_i, v_1)\) for each observation \(x_i\). Here, \(\beta_0 = 1\) is a constant term, \(\delta = 0.5\) is a coefficient modulating the impact of \(x_i\), and \(v_1 = 0.1\) is the variance added to encapsulate uncertainty in the mean responses. Moreover, \(\Sigma_\mathbf{Y} \in \mathbb{R}^{q \times q}\) is a symmetric, positive definite, and invertible matrix designed to capture inter-label correlations, thus preserving the inherent dependency structure between \(\mathbf{Y}\) responses.

The \(\mathbf{Y}\) responses are random objects evolving in a metric space \(\Omega\) equipped with the Wasserstein distance. Their distribution is described by a quantile function \(Q(\mathbf{Y})\), representing the distribution of \(\mathbf{Y}\) responses in this space. The distance between two objects in this space is measured using the Wasserstein distance, defined for two distributions \(P\) and \(Q\) as follows:
\[
d^2_W(P, Q) = \|\mu_P - \mu_Q\|^2 + \text{Tr}\left(\Sigma_P + \Sigma_Q - 2 \left(\Sigma_P^{1/2} \Sigma_Q \Sigma_P^{1/2}\right)^{1/2}\right),
\]
where \(\mu_P\) and \(\mu_Q\) are the means of the \(P\) and \(Q\) distributions, and \(\Sigma_P\) and \(\Sigma_Q\) are their respective covariance matrices.

To make the solution tractable, we discretize the metric space \(\Omega\) using a uniform grid \(\{u_j\}_{j=1}^M\) on the interval \([0, 1]\). The minimization then becomes a quadratic optimization problem, where we seek the quantile function \(Q(\omega)\) by solving the following problem:
\[
q^\ast = \arg \min_{q \in \mathbb{R}^M} \|g - q\|^2_E,
\]
under the constraint \(q_1 \leq \cdots \leq q_M\), where \(g_j = \widehat{g}_x(u_j)\). This approach provides a discretized approximation of the quantile function \(Q(\omega)\), representing the solution of the global regression problem in the metric space \(\Omega\).

Using this discretization, the estimator \(\widehat{\zeta}_\oplus(x)\) is defined as follows:
\[
\widehat{\zeta}_\oplus(x) = \arg \min_{w \in \Omega} \sum_{i=1}^{n} \left[ 1 + (x_{i} - \overline{\mathbf{X}})^\top \widehat{\Sigma}^{-1} (x_{i} - \overline{\mathbf{X}}) + (x_{i} - \overline{\mathbf{X}})^\top \widehat{\Sigma}_{\mathbf{X}\mathbf{Y}}^{-1} (y_{i} - \overline{\mathbf{Y}})^\top \right] d^2_W(\mathbf{Y}, \omega)
\]

\subsubsection*{\textbf{Interpretation}} Figure \ref{fig1:was} compare the performance of the estimators \(\widehat{\Lambda}_\oplus(x), \widehat{\zeta}_\oplus(x)\) and \(\hat{m}_\oplus(x)\), respectively, using Wasserstein distances. The losses associated with the three estimators allow the following conclusions to be drawn: \( \widehat{\Lambda}_\oplus(x) \), which results from the application of a denoising and multicollinearity reduction, shows the best performance. It displays losses close to zero for the majority of sample sizes, which testifies to its stability and precision. This makes it the most reliable option among the three. Before the application of the denoising and multicollinearity reduction, \( \widehat{\zeta}_\oplus(x) \) displays a wider range of losses, although overall close to zero, indicating that it remains a good estimator, but slightly less stable than \( \widehat{\Lambda}_\oplus(x) \). Finally, \( \hat{m}_\oplus(x) \) shows more varied losses, especially in the negative part, suggesting that it is less robust and less precise than the other two estimators. In fact, integrating the term \( (x_{i} - \overline{\mathbf{X}})^\top \widehat{\Sigma}_{\mathbf{X}\mathbf{Y}}^{-1} (y_{i} - \overline{\mathbf{Y}})^\top \) into \( \widehat{\zeta}_\oplus(x) \) allows a significant reduction in loss compared to \( \hat{m}_\oplus(x) \), demonstrating the importance of this interaction term in improving the model fit. After applying denoising and multicollinearity reduction, \( \widehat{\Lambda}_\oplus(x) \) benefits from a further improvement, further reducing the loss and making the estimation more precise. This improvement is particularly marked when the sample size \( n \) increases, suggesting that the inclusion of the interaction term and the selection of variables allow to better capture the complex relationships between the explanatory variables \( \mathbf{X} \) and the responses \( \mathbf{Y} \). Thus, this approach improves not only the estimation, but also the stability of the estimator, especially in a context where the dependencies between the variables are complex and must be modeled more precisely.

\begin{figure}[h]
\centering
\includegraphics[width=0.9\linewidth]{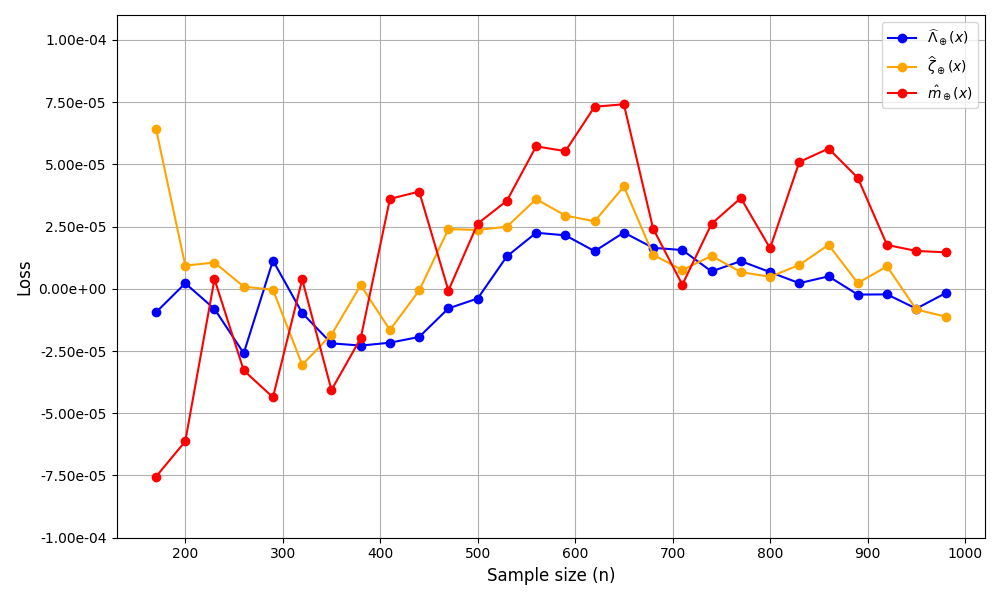}
\caption{Loss curves for the \(\widehat{\Lambda}_\oplus(x), \widehat{\zeta}_\oplus(x)\) and \(\hat{m}_\oplus(x)\) estimators using the Wasserstein distance (Probability Distributions).}
\label{fig1:was}
\end{figure}


\subsection{\textbf{Multilabel Fréchet Regression with Spherical Data}}
\subsubsection*{\textbf{Computation details}}
The explanatory variables \(\mathbf{X}\) follow a multivariate normal distribution \(\mathbf{X} \sim \mathcal{N}(\mu_\mathbf{X}, \Sigma_\mathbf{X})\), with \(\mu_\mathbf{X} = [0, 1, -1, 2, 0]\) and \(\Sigma_\mathbf{X}\) being positive definite and invertible. These data are generated in spherical form, which means that the covariance matrix \(\Sigma_\mathbf{X}\) is chosen to give an isotropic (spherical) distribution after transformation. The responses \(\mathbf{Y}\) are modeled by \(\mathbf{Y} \sim \mathcal{N}(\beta(x_i) \mathbf{1}_q, \Sigma_\mathbf{Y})\), where \(\beta(x_i) \sim \mathcal{N}(\beta_0 + \delta x_i, v_1)\), with \(\beta_0 = 1\), \(\delta = 0.5\) and \(v_1 = 0.1\). The covariance matrix \(\Sigma_\mathbf{Y}\) is symmetric and positive definite, capturing the correlations between the responses. The responses \(\mathbf{Y}\) evolve in a metric space \(\Omega\) equipped with the Wasserstein distance.
\subsubsection*{\textbf{Interpretation}} Figure \ref{fig2:was} compare the performance of the estimators \(\widehat{\Lambda}_\oplus(x), \widehat{\zeta}_\oplus(x)\) and \(\hat{m}_\oplus(x)\), respectively, using Wasserstein distances. 
The results presented show the losses associated with the three estimators for different sample sizes $n$. $\widehat{\Lambda}_\oplus(x)$ displays losses close to zero for most sample sizes, indicating high stability and accuracy, with minimal variations, especially for larger sizes. In contrast, $\widehat{\zeta}_\oplus(x)$ displays more scattered losses, oscillating between positive and negative values, suggesting a less stable fit than $\widehat{\Lambda}_\oplus(x)$, although it remains competitive. Finally, $\hat{m}_\oplus(x)$ shows more varied losses, often negative, which makes it less robust and accurate than the other estimators, with more uneven performances, especially for certain sample sizes. In conclusion, $\widehat{\Lambda}_\oplus(x)$ stands out for its stability and low losses, making it the most reliable estimator, while $\widehat{\zeta}_\oplus(x)$ and $\hat{m}_\oplus(x)$ show lower accuracy.

\begin{figure}[h]
\centering
\includegraphics[width=0.9\linewidth]{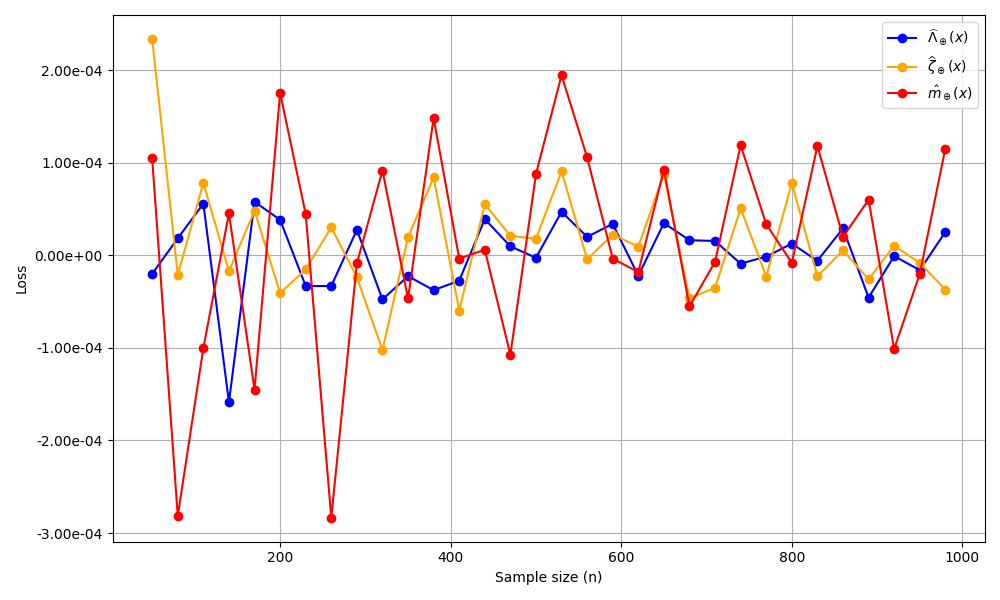}
\caption{Loss curves for the \(\widehat{\Lambda}_\oplus(x), \widehat{\zeta}_\oplus(x)\) and \(\hat{m}_\oplus(x)\) estimators using the Wasserstein distance (Spherical Data).}
\label{fig2:was}
\end{figure}

\section{Conclusion}
In this paper, we address two critical challenges in Fréchet regression for complex, high-dimensional data. First, we propose an extension of the Global Fréchet regression model, enabling explicit modeling of relationships between input variables and multiple responses. This advancement is particularly significant for multi-label settings, where intricate dependencies between predictors and responses are common. Second, we introduce a novel framework to mitigate the effects of noise and multicollinearity. By leveraging implicit regularization, our approach preserves the intrinsic structure of the data while effectively capturing complex dependencies, offering a robust and efficient alternative to traditional explicit regularization methods, which often introduce bias and complicate optimization. Theoretical analysis and numerical experiments demonstrate the effectiveness of our method, highlighting its potential for applications in fields such as imaging and bioinformatics. Future work could explore extensions to other data spaces, as well as scalability to even larger datasets, further broadening the applicability and versatility of our framework.


\bibliographystyle{imsart-nameyear} 
\bibliography{bibliography}       







\end{document}